\documentclass[runningheads]{llncs}
\usepackage{times}
\usepackage{xcolor}
\usepackage{soul}
\usepackage[utf8]{inputenc}
\usepackage[small]{caption}
\usepackage{subfigure}
\usepackage{graphicx}
\usepackage{amsmath}
\usepackage{amsfonts}
\usepackage{url}
\newcommand{\tabincell}[2]{\begin{tabular}{@{}#1@{}}#2\end{tabular}}  

\begin{document}

\title{MOBA-Slice: A Time Slice Based\\Evaluation Framework of Relative Advantage\\between Teams in MOBA Games}

\titlerunning{MOBA-Slice: An Evaluation Framework of MOBA Games}

\author{Lijun Yu\inst{1,2} \and
Dawei Zhang\inst{2} \and
Xiangqun Chen\inst{1} \and
Xing Xie\inst{2}}

\authorrunning{L. Yu et al.}

\institute{Peking University
\email{yulijun@pku.edu.cn, cherry@sei.pku.edu.cn} \and
Microsoft Research Asia
\email{zhangdawei@outlook.com, xing.xie@microsoft.com}
}
\maketitle

\begin{abstract}
Multiplayer Online Battle Arena (MOBA) is currently one of the most popular genres of digital games around the world. The domain of knowledge contained in these complicated games is large. It is hard for humans and algorithms to evaluate the real-time game situation or predict the game result. In this paper, we introduce \textbf{MOBA-Slice}, a time slice based evaluation framework of relative advantage between teams in MOBA games. MOBA-Slice is a quantitative evaluation method based on learning, similar to the value network of AlphaGo. It establishes a foundation for further MOBA related research including AI development. In MOBA-Slice, with an analysis of the deciding factors of MOBA game results, we design a neural network model to fit our discounted evaluation function. Then we apply MOBA-Slice to Defense of the Ancients 2 (DotA2), a typical and popular MOBA game. Experiments on a large number of match replays show that our model works well on arbitrary matches. MOBA-Slice not only has an accuracy 3.7\% higher than DotA Plus Assistant\footnote{A subscription service provided by DotA2} at result prediction, but also supports the prediction of the remaining time of a game, and then realizes the evaluation of relative advantage between teams.
\keywords{Computer Games \and Applications of Supervised Learning \and Game Playing and Machine Learning}
\end{abstract}

\section{Introduction}
Multiplayer Online Battle Arena (MOBA) is a sub-genre of strategy video games. Players of two teams each control a playable character competing to destroy the opposing team's main structure, with the assistance of periodically spawned computer-controlled units. Figure \ref{fig:mobamap}\footnote{\url{https://en.wikipedia.org/wiki/File:Map_of_MOBA.svg}} is a typical map of a MOBA genre game. MOBA is currently one of the most popular genres of digital games around the world. Among championships of MOBA globally, Defense of the Ancients 2 (DotA2) has the most generously awarded tournaments. DotA2 is a typical MOBA game in which two teams of five players collectively destroy enemy's structure, Ancient, while defending their own. The playable characters are called heroes, each of which has its unique design, strengths, and weaknesses. The two teams, Radiant and Dire, occupy fortified bases called plateau in opposite corners of the map as Figure \ref{fig:dotamap}\footnote{\url{https://dota2.gamepedia.com/File:Minimap_7.07.png}} shows.

\begin{figure}[htb!]
\centering
\subfigure[A Typical MOBA Map]{
\includegraphics[width=5.5cm]{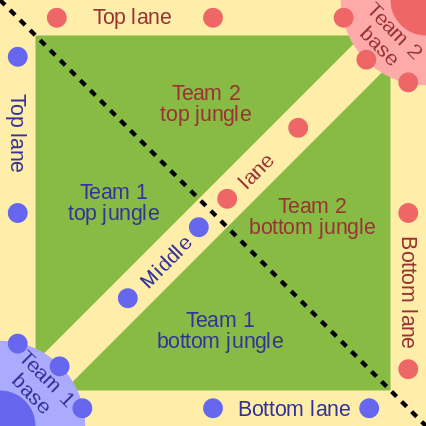}
\label{fig:mobamap}}
\subfigure[DotA 2 Mini Map]{
\includegraphics[width=5.5cm]{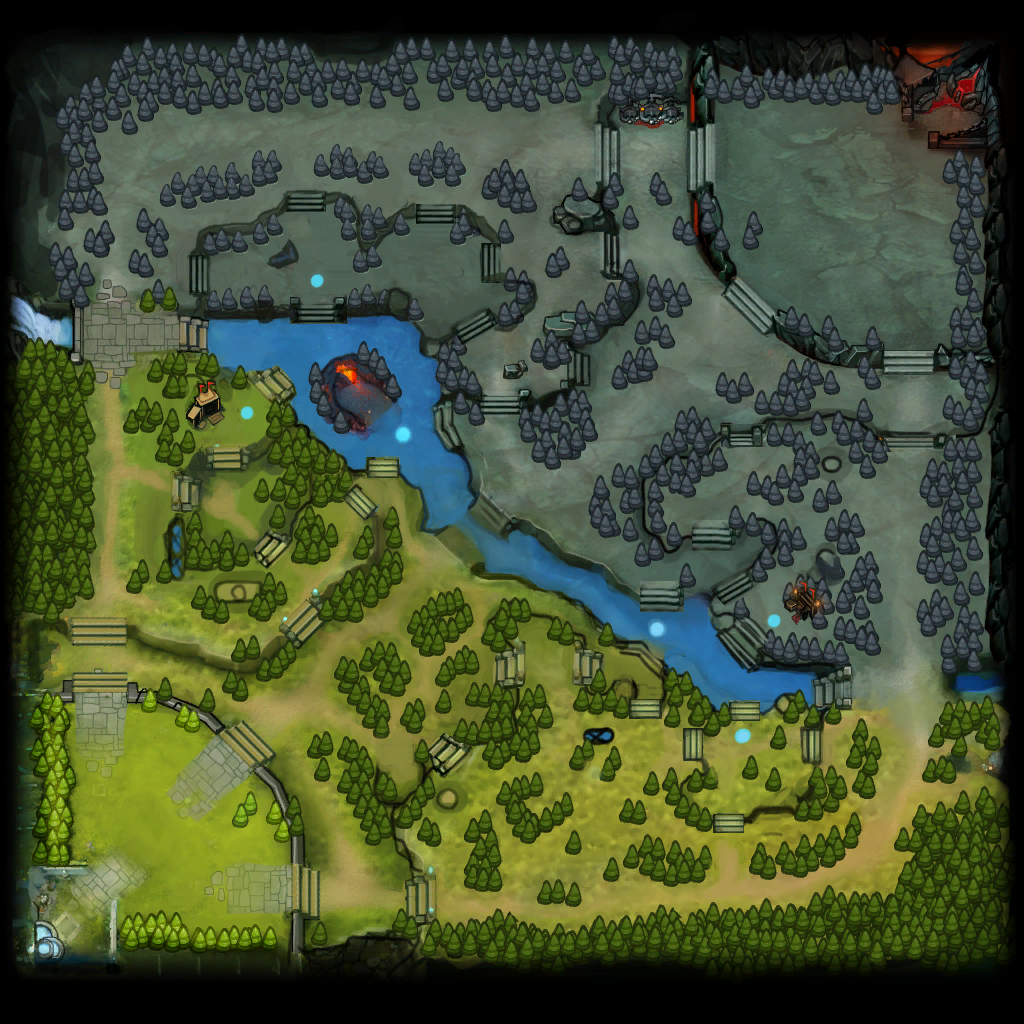}
\label{fig:dotamap}}
\caption{Maps}
\end{figure}

In games with a scoring mechanism, we can easily tell which player or team has an advantage from the scores. But the design of MOBA games such as DotA2 is complicated, with lots of variables changing during the whole game. So it is hard to evaluate the real-time game situation in such a large domain of knowledge. Traditionally, players and commentators assess the relative advantage by intuitive feeling, their own experience and fuzzy methods. No unified standard has been proposed, to the best of our knowledge. Such evaluation is needed in further research related to MOBA. It plays an essential role in developing artificial intelligence for MOBA games, such as working as the reward function in reinforcement learning models~\cite{Sutton:1998:IRL:551283} or the evaluation function in Monte Carlo planning models~\cite{Chung2005Monte}. In strategy analysis, the effectiveness of strategies can also be estimated by the change of relative advantage between teams.

In this paper, we introduce MOBA-Slice, which is able to evaluate any time slice of a game quantitatively. Different from manually designed evaluation function, MOBA-Slice provides a model that learns from data, which is similar to the value network of AlphaGo~\cite{silver2016mastering}. It establishes a foundation for further MOBA related research including AI development and strategy analysis.

The main contribution of this paper is listed below.
\begin{enumerate}
\item \emph{We introduce MOBA-Slice, a time slice based evaluation framework of relative advantage between teams in MOBA games.} We analyze the deciding factors of MOBA game result. A discounted evaluation function is defined to compute the relative advantage. We design a supervised learning model based on Neural Network to do this evaluation. MOBA-Slice is able to predict the result and remaining time of ongoing matches.
\item \emph{We apply MOBA-Slice to DotA2 and prove the effectiveness of MOBA-Slice with experiments.} We embody MOBA-Slice on a typical MOBA game, DotA2. We process a large number of DotA2 match replays to train our model. Experiments show that the model is able to evaluate time slices of arbitrary DotA2 matches. In the aspect of predicting the game result, MOBA-Slice has an accuracy 3.7\% higher than DotA Plus Assistant.
\end{enumerate}

\section{MOBA-Slice}
\subsection{MOBA Game Result Analysis}
In a MOBA game, the final victory is of the most significance. MOBA Game Result (\textbf{MGR}) analysis is defined to describe the deciding factors of the result of a match. For a certain time point, the future result of the game is considered related to two aspects, current state and future trend. The current state describes the game situation at this specific time slice, which is the foundation of future development. The future trend represents how the match will develop from the current state. Figure \ref{fig:mgr} shows the content of MGR analysis.
\begin{figure}[htb!]
\centering
\includegraphics[width=10cm]{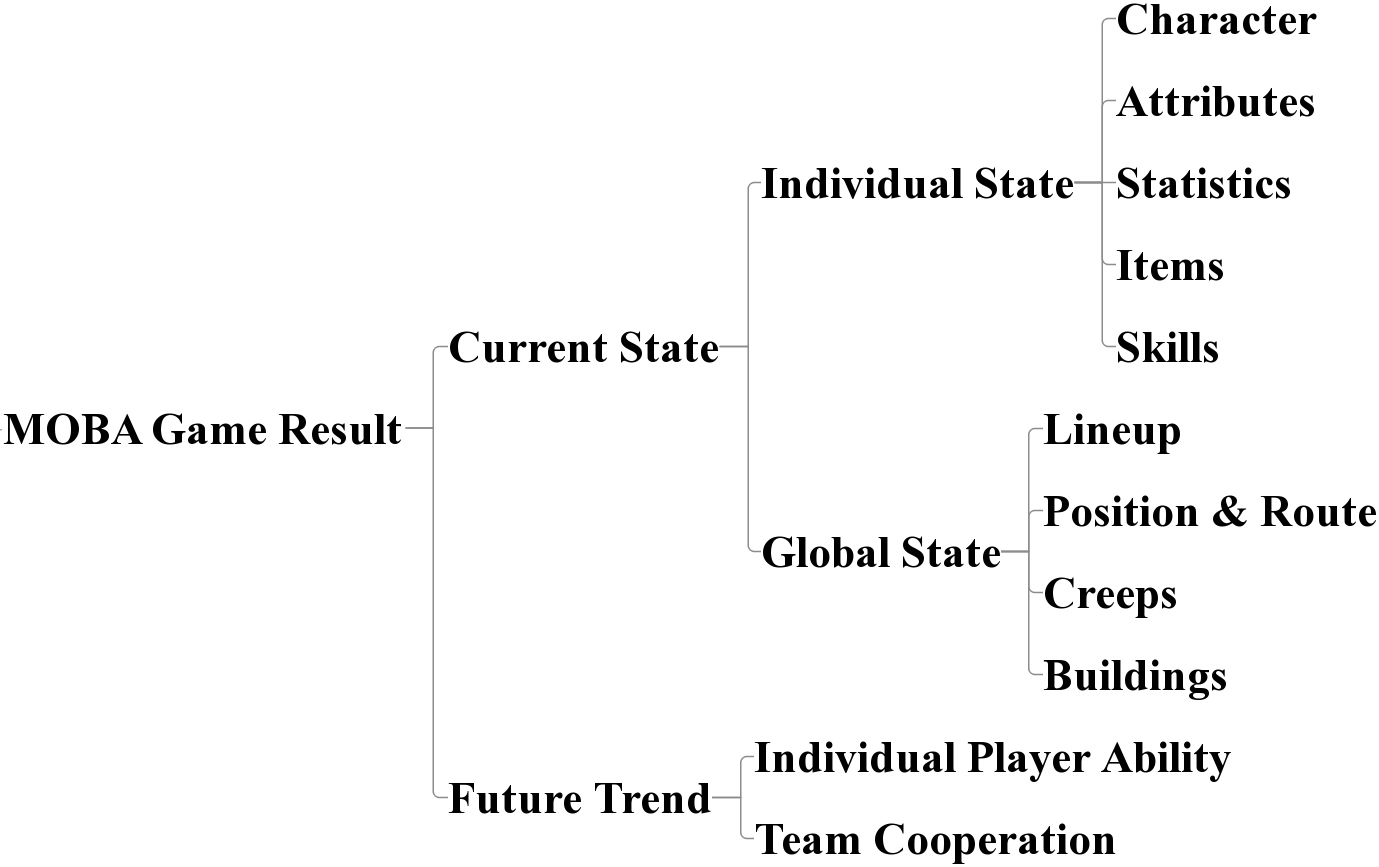}
\caption{MGR Analysis}
\label{fig:mgr}
\end{figure}

In details, the current state includes every information we can get to know about the game from a particular time slice. As MOBA are multiplayer games where each player controls one character, we further divide the factors of current status into individual and global state. The individual state contains the information of each character, including its type, attributes such as position and experience, statistics such as the count of deaths, the items it possesses, the skills it masters, and so on. The global state describes the background environment and the relationship between characters. The global information includes the lineup depicting the types of characters in both teams. Lineup can also reveal the advantage or disadvantage of a specific combination of characters. The position and selected route of characters, the status of creeps and buildings are considered in global state, as well. All the information mentioned above that will be used to represent current state can be found in game records, such as replay files.

Different from the current state, the part of future trend focuses on the level of players. Given the same status as a start, different players and teams will have distinct game processes and lead to various game results. We also divide this part into individual ability and team cooperation. The ability of each player individually can be measured through ranking systems such as TrueSkill~\cite{Sch2007TrueSkill,Dangauthier2007TrueSkill}. We further include the part of team cooperation, as good members do not necessarily mean a strong team. The detailed method to evaluate team cooperation still requires further research and is not covered in this paper.

\subsection{Discounted Evaluation}
MGR analysis describes how the result of a match comes out, and reversely we can use the \emph{future} result to evaluate the \emph{current} states. In Q learning~\cite{Watkins1992}, the discount factor $\gamma$ represents the difference in importance between future rewards and present rewards. We can let $\gamma=\frac{1}{1+r}$, where $r$ is the discount rate. The current value of a reward $R_a$ after time $t$ is calculated by $\gamma^tR_a$. Inspired by this, in the evaluation of current situation, we regard the future victory as a reward. The farther we are from the victory, the less the current value of its reward is. We use the logarithmic form of discount factor to simplify exponent operation.
 
\begin{definition}
The function of discounted evaluation $\mathit{DE}$ for a time slice $\mathit{TS}$ is defined as:
\begin{equation}\label{equ:f}
\mathit{DE_{TS}}(R, t) = \frac{1}{\ln(1+r)^t} \times R = \frac{R}{\alpha t}
\end{equation}
where $\alpha = \ln(1+r$), $t$ is the remaining time of the game, and
\begin{equation}
R = \left\{
\begin{aligned}
1, \mbox{when Team A wins} \\
-1, \mbox{when Team B wins}
\end{aligned}
\right.
\end{equation}
\end{definition}
$\mathit{DE_{TS}}$ has several appealing properties:
\begin{enumerate}
\item The sign of its value represents the final result, positive for A's victory, negative for B's.
\item Its absolute value is inversely proportional to $t$.
\item The value approximately represents the advantage team A has in comparison to team B.
\end{enumerate}

Here is some explanation of property 3. In normal cognition, greater advantage indicates more probability of winning, which would result in less time to end the game. So we can suppose that there is a negative correlation between the advantage gap between two teams and $t$. This assumption may be challenged by the intuition that the game situation fluctuates and even reverses at times. We can divide such fluctuation into random and systematic ones. The players' faults that occur randomly can be smoothed in a large scale of data. Systematic fluctuation is considered to be the result of misinterpretation of advantage. For example, in games where team A wins, if the intuitionistic advantage of team A usually gets smaller after certain time slice $x$, there is a reason to doubt that the advantage of team A at $x$ is over-estimated. Traditional ways to evaluate advantage usually ignores the potential fluctuation and does not take game progress into consideration. In our method, we can suppose the absolute value of the advantage between teams keeps growing from the beginning as a small value till the end. The value of function $\mathit{DE_{TS}}$ changes in the same way.

Although the values of both $R$ and $t$ are unknown for a current time slice in an ongoing match, we can train a supervised learning model to predict the value of function $\mathit{DE_{TS}}$ for each time slice. Based on property 3, the model would be a good evaluator for relative advantage between teams.

\subsection{Time Slice Evaluation Model}
We intend to design a supervised learning model which takes time slices as input and outputs the value of function $\mathit{DE_{TS}}$.  Due to the current limitation of research, we only consider the factors in the part of current state in MGR analysis while designing models.

The structure of this Time Slice Evaluation (\textbf{TSE}) model is shown in Figure \ref{fig:tse}. TSE model contains two parts of substructures. The individual (Ind) part calculates the contribution of each character in the game separately, which corresponds to the individual state in MGR analysis. Different characters in MOBA games have distinctive design, strengths, and weaknesses. This part ignores any potential correlation between characters but learns the unique features of each character. The global (Glo) part calculates the contribution of all the characters in a match, corresponding to the global state in MGR analysis. This part takes all the characters in a match as a whole and is designed to learn the potential relationship of addition or restriction. To combine the Ind and Glo, the outputs of the two parts are fed to $l_c$ layers of $n_c$ neurons activated by $relu$ function~\cite{DBLP:journals/corr/abs-1710-05941}. The output is calculated by one neuron activated by $tanh$ function to get a result in the range $[-1, 1]$.

For a MOBA game which has $c_a$ playable characters in total, the Ind part consists $c_a$ parts of subnets, each of which calculates a character's contribution. For $c_m$ characters in a time slice, we use their corresponding subnets to calculate their contribution individually, sum for each team and then subtract team B from team A to get the final result. Usually, in MOBA games each character is not allowed to be used more than once, so each subnet is calculated at most once for a time slice. But in practice, there is still a problem for Ind part to train in batch. Different characters are used in different matches, so the outputs need to be calculated with different sets of subnets. The data path in the model varies for time slices from different matches.

\begin{figure}[htb!]
\centering
\subfigure[TSE Model]{
\includegraphics[height=5cm]{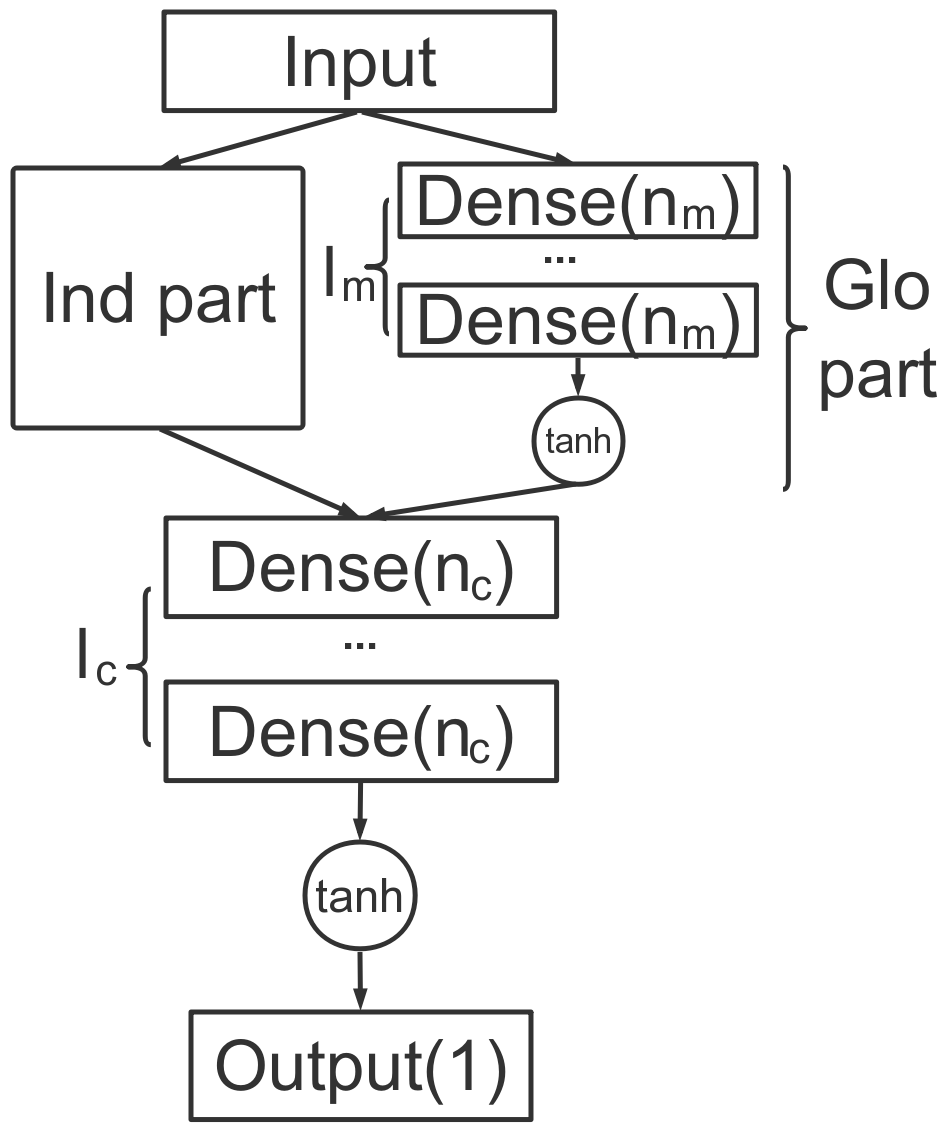}
\label{fig:tse}}
\subfigure[Subnet]{
\includegraphics[height=4cm]{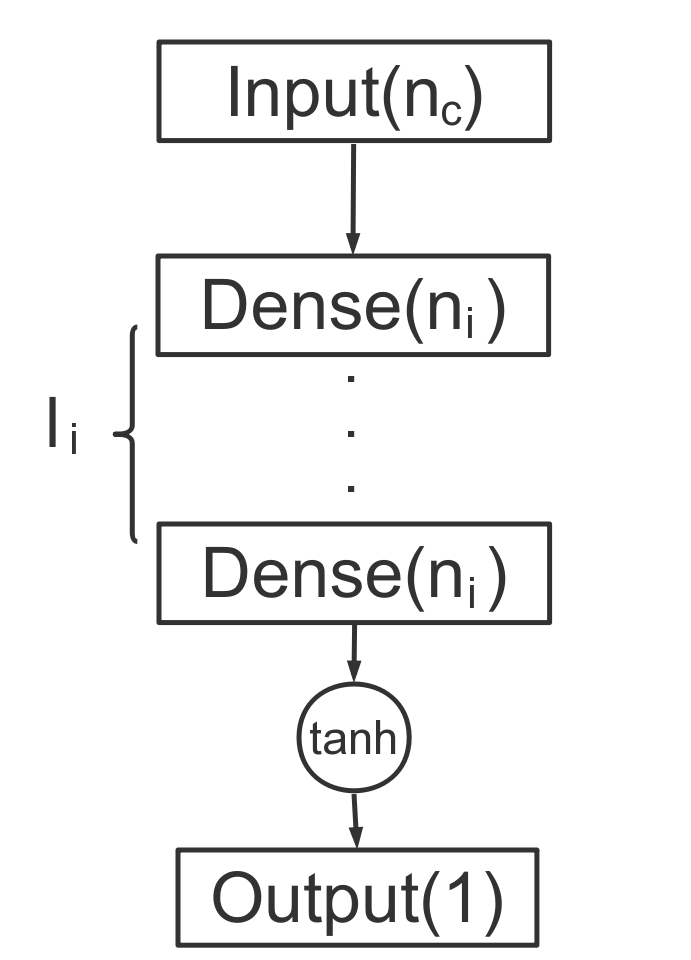}
\label{fig:fnn}}
\subfigure[Ind Part]{
\includegraphics[width=9cm]{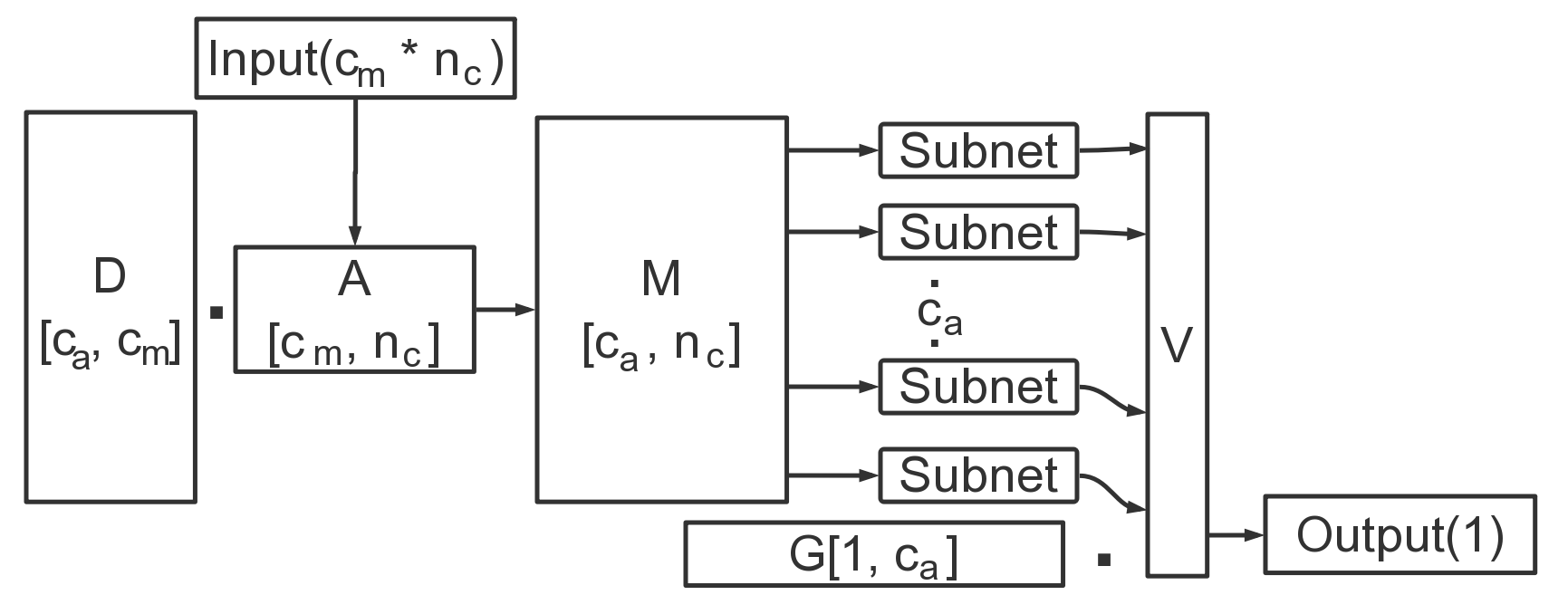}
\label{fig:ind}}
\caption{Model Structure}
\end{figure}

With the structure defined in Figure \ref{fig:ind}, batch training is implemented. If each character takes $n_c$ dimensions of data, the input of a time slice can be reshaped into a matrix $A \in \mathbb{R}^{c_m \times n_c}$. We create a distributer matrix $D$ for each input, which only contains $c_m$ non-zero elements. $D[i, j] = 1$ means that the id of the $i$-th character in this game is $j$, so the $i$-th row of $A$ will be at the $j$-th row in matrix $M=D \cdot A$ and will be fed to the $j$-th subnet. In vector $G \in \mathbb{R}^{1 \times c_a}$, $G[0, i]=1$ indicates character $i$ is in team A and $G[0, i]=-1$ indicates character $i$ is in team B. Vector $V \in \mathbb{R}^{c_a \times 1}$ contains the output of all $c_a$ subnets. We finally calculate the output by $\mathit{output}=G \cdot V$.

The subnets of Ind part are feed-forward neural networks of the same scale. Each subnet in Figure \ref{fig:fnn} takes in $n_c$ dimensions of input and outputs with a neuron activated by $\tanh$. Each subnet has $l_i$ hidden layers of $n_i$ neurons activated by $\mathit{relu}$ and applied dropout~\cite{srivastava2014dropout:}. at rate $r_d$

The Glo part is simply a multilayer feed-forward structure. It is like the subnet of Ind part but in a much larger scale. It takes in a full time slice vector and outputs with one neuron activated by $\tanh$. The Glo part has $l_m$ fully connected hidden layers with each $n_m$ neurons, activated by $\mathit{relu}$ function and applied dropout at rate $r_d$.

To get better distribution, we set the output of TSE model:
\begin{equation}
y = \frac{1}{\mathit{DE_{TS}}(R, t)} = \frac{\alpha t}{R}
\end{equation}
And it needs to be transformed to correspond to the range of $\tanh$ function:
\begin{equation}
y_\mathit{scaled} = -1 + 2 \times \frac{y - y_{min}}{y_{max} - y_{min}} \in [-1, 1]
\end{equation}
Let $\hat{y}$ be the prediction from model. To transform it back. we can rescale $\hat{y}$ by:
\begin{equation}
\hat{y}_\mathit{rescaled} = y_{min} + \frac{\hat{y} + 1}{2}\times(y_{max}-y_{min})
\end{equation}
and then prediction of $t$ and $R$ can be extracted as:
\begin{equation}
\hat{t}=\left.\left|\hat{y}_\mathit{rescaled}\right|\middle/\alpha\right.
\end{equation}
\begin{equation}
\hat{R}=\mathit{sign}(\hat{y}_\mathit{rescaled})
\end{equation}

As a regression problem, mean absolute error (MAE) and mean squared error (MSE) are chosen as metrics. MAE is also the loss function. We can further calculate a rescaled MAE by
\begin{equation}
\mathit{MAE}_\mathit{rescaled}(\hat{y}, y_\mathit{scaled}) = \mathit{MAE}(\hat{y}, y_\mathit{scaled}) \times \frac{y_{max} - y_{min}}{2}
\end{equation}

\begin{lemma}
\begin{equation}
\left|\hat{y}_\mathit{rescaled} - y\right| \ge \alpha\left|\hat{t} - t\right|
\end{equation}
\end{lemma}

\begin{proof}
\begin{equation}
\begin{aligned}
\left|\hat{y}_\mathit{rescaled} - y\right| &= \left|\frac{\alpha\hat{t}}{\hat{R}} - \frac{\alpha t}{R}\right| \\
&=\begin{cases}
\alpha\left|\hat{t} - t\right| & \mbox{when } \hat{R}=R \\
\alpha\left|\hat{t} + t\right| \ge \alpha\left|\hat{t} - t\right| & \mbox{when } \hat{R}=-R
\end{cases}
\end{aligned}
\end{equation}
where $\left|R\right|=\left|\hat{R}\right|=1$ and $t, \hat{t}\ge0$
\end{proof}

\begin{theorem}
\begin{equation}
\mathit{MAE}_\mathit{rescaled}(\hat{y}, y_\mathit{scaled}) \ge \alpha\mathit{MAE}(\hat{t}, t)
\end{equation}
\end{theorem}

\begin{proof}
\begin{equation}
\begin{aligned}
\mathit{MAE}_\mathit{rescaled}(\hat{y}, y_\mathit{scaled}) &= \mathit{MAE}(\hat{y}_\mathit{rescaled}, y) \\
&= \frac{\sum_{i=1}^{N}{\left|\hat{y}_\mathit{rescaled} - y\right|}}{N} \\
&\ge \frac{\sum_{i=1}^{N}{\alpha\left|\hat{t} - t\right|}}{N} \\
&=\alpha\mathit{MAE}(\hat{t}, t)
\end{aligned}
\end{equation}
\end{proof}
So $\left.\mathit{MAE}_\mathit{rescaled}(\hat{y}, y_\mathit{scaled})\middle/\alpha\right.$ proves to be the upper bound of $\mathit{MAE}(t, \hat{t})$. It provides a more intuitive way to evaluate the model's effectiveness, as its value can be viewed in units of time to reveal the mean error of prediction.

\section{Experiments}
\subsection{Apply MOBA-Slice to DotA2}
We choose DotA2 as a typical MOBA game to apply MOBA-Slice. DotA2 generates a replay file to record all the information in a match. An open source parser from OpenDota project\footnote{Replay parser from OpenDota project: \url{https://github.com/odota/parser}} can parse the replay file and generate \emph{interval} messages every second to record the state of each character. The following information contained in \emph{interval} messages is chosen to describe a time slice in current experiments.

\begin{itemize}
	\item Character - hero id
	\item Attributes: life state, gold, experience, coordinate(x, y)
	\item Statistics: 
	\begin{itemize}
		\item deaths, kills, last hit, denies, assists 
		\item stacked creeps, stacked camps, killed towers, killed roshans
		\item placed observer, placed sentry, rune pickup, team-fight participation
	\end{itemize}
	\item Items: 244 types
\end{itemize}

There are 114 characters which are called heroes in DotA2. In a match, each team chl4
ooses five heroes without repetition. A hero may die at times and then respawn after a period of delay. The life state attribute is used to record whether a hero is alive or dead. Gold and experience are two essential attributes of a hero. Gold is primarily obtained by killing enemy heroes, destroying enemy structures and killing creeps. Heroes can use their gold to purchase items that provide special abilities. A hero's experience level begins at level 1 and grows as the game goes on. It is highly related to the level of a hero's ability. The position of a hero on the map is given in coordinate representation x and y. 

Many kinds of statistics are generated automatically by the game engine and the replay parser to help analyze the game. Deaths, kills, last hit, denies and assists record these historical events during the fights. Stacked creeps, camps and killed towers, roshans record the hero's fight history regarding these units. Invisible watchers, observer and sentry, are helpful items that allow watching over areas and spy on enemies. Runes are special boosters that spawn on the game map, which enhance heroes' ability in various ways. A team-fight is a fight provoked by several players with a considerable scale. There are 244 types of items in DotA2 that can be purchased by a hero, according to our statistics. 

In the aspect of global state, lineup is represented by the id of 10 heroes. Position and route are reflected in the coordinates of heroes. Since skill of heroes and status of creeps and buildings are not recorded in \emph{interval} messages, we do not involve these fields of information in current experiments. Using lower-level parser such as manta\footnote{Manta: \url{https://github.com/dotabuff/manta}} and clarity\footnote{Clarity 2: \url{https://github.com/skadistats/clarity}} to extract data from raw Protobuf structured replay is likely to address this problem, but it significantly increases the complexity of data processing.

\subsection{Data Processing}
The digital distribution platform developed by Valve Corporation, Steam, officially provides APIs\footnote{Documentation of Steam APIs for DotA2: \url{https://wiki.teamfortress.com/wiki/WebAPI#Dota_2}} for DotA2 game statistics. We use \emph{GetLeagueListing} and \emph{GetMatchHistory} methods to get the list of all matches of professional game leagues\footnote{Data processing took place in Oct. 2017}. We then use OpenDota's API\footnote{Documentation of OpenDota API for match data: \url{https://docs.opendota.com/#tag/matches}} to get detailed information of matches including the URL of its replay file on Valve's servers. After downloading and decompressing them, we send replay files to OpenDota's parser. With a massive cost of CPU time, we get downloaded and parsed replay files of 105,915 matches of professional leagues, with a total size of over 3 TB.

We scan parsed replay files and create time slice vectors every 60 seconds of game time. In a vector, each of ten heroes has 263 dimensions, including 1 for hero id, 5 for attributes, 13 for statistics and 244 for items. Along with one dimension recording the game time, each time slice vector has 2,631 dimensions. We also generate the value of function $\mathit{DE_{TS}}$ for each time slice, with the already known game result. In this process, we notice that about 34 thousand replays are not correctly parsed due to game version issues, file corruption or known limitation of the parser. After dropping these invalid data, we get 2,802,329 time slices generated from 71,355 matches. The distribution of match length is shown in Figure \ref{fig:len}. The average length is about 40 minutes, with few matches shorter than 10 minutes and very few matches longer than 100 minutes. 

\begin{figure}[htb!]
\centering
\includegraphics[width=9cm]{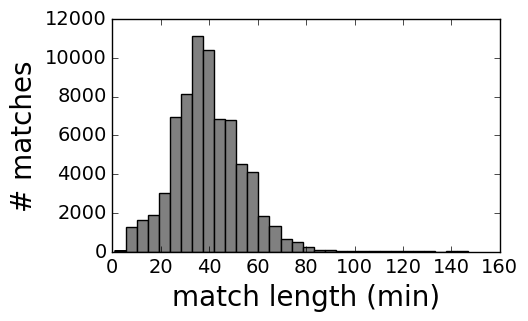}
\caption{Distribution of Match Length}
\label{fig:len}
\end{figure}

\subsection{First Training of TSE Model}
Considering that the beginning part of a game might be erratic, we choose to only use time slices from the last 50\% time in each match in the first experiment. Experiments with whole matches and other range of game time will be talked about in Section \ref{sec:tbp}. To simplify and get better distribution, we set $r = e - 1$ so $\alpha = 1$. \emph{Adam} optimizer~\cite{kingma2015adam:} is applied in the experiments.

To better know the optimized hyper-parameters of TSE model, we begin with training and testing partially. We train separately the Ind part with the input described above and $y_{scaled}$ as its output. Each subnet of the Ind part takes in $n_c=263$ dimensions representing one hero, including 1 for game time, 5 for attributes, 13 for statistics and 244 for items, which is similar to one-tenth of a time slice vector except that hero id is replaced. Due to its structure, only $\frac{c_m}{c_a}$ of the Ind part is trained with every piece of data, so the training process is slower than a fully connected model. For subnets, $l_i=3, n_i=40, r_d=0.5$ are proved to be the best, according to our experiments. We use 90\% of the dataset for training, 5\% for validation and 5\% for testing. This trained Ind part shows an MAE of $0.1523$ which indicates a prediction error of $11.15$ minutes.

For the Glo part, we take the time slice vectors directly as input and $y_{scaled}$ as output. $n_m=400, l_m=4, r_d=0.5$ are chosen as the hyper-parameters of the best performance. On the same dataset as the Ind part, the Glo part shows a rescaled MAE of $7.85$ minutes.

For the remaining part of TSE model, $n_c=4, l_c=3$ are set. When we train the whole TSE model, we regard it as a multi-output model to avoid training different parts unequally. The loss is calculated by Equation \ref{equ:loss}. $\hat{y}_{Ind}$, $\hat{y}_{Glo}$ and $\hat{y}$ are the outputs of Ind part, Glo part, and the whole model.
\begin{equation}
loss=\mathit{MAE}(y, \hat{y}) + \mu \times \mathit{MAE}(y, \hat{y}_{Ind}) + \nu \times \mathit{MAE}(y, \hat{y}_{Glo})
\label{equ:loss} 
\end{equation}
$\mu=0.3, \nu=0.3$ are proved to be an efficient set of parameters in this training. The whole model is trained on the same dataset as previous at first. Then a 10-fold cross validation at match level is applied to provide reliable result.

\begin{table}[htb!]
\centering
\caption{Metrics of TSE Model}
\begin{tabular}{|c|c|c|c|}
\hline
 & $\mathit{MAE}$ & $\mathit{MSE}$ & Rescaled $\mathit{MAE}$(minutes) \\ \hline
Blind prediction & 0.5178 & 0.3683 & 37.91 \\ \hline
Ind part & 0.1523 & 0.0339 & 11.15 \\ \hline
Glo part & 0.1072 & 0.0290 & 7.85\\ \hline
TSE model & 0.1050 & 0.0287 & 7.69 \\ \hline \hline
\tabincell{c}{TSE model\\(10-fold cross validation)} & 0.10539 & 0.02794 & 7.716\\ \hline
\end{tabular}
\label{tab:res}
\end{table}

The result of experiments on partial and full TSE model is shown in Table \ref{tab:res}. The performance of TSE model indicates that it is an applicable model to fit function $\mathit{DE_{TS}}$. Both the Ind part and the Glo part are valid compared with blind prediction. The Glo part has a better performance than the Ind part. That means only focusing on the individual contribution separately would result in a significant loss of information. The correlation between heroes plays an essential role in the matches. TSE model combines the two parts and is better than either part, so the validity of both Ind part and Glo part is proven. As TSE model is designed based on the split of individual and global state in MGR analysis, its performance also supports MGR analysis.

\subsection{Prediction on Continuous Time Slices}

In the previous experiment, time slices are randomly shuffled. We can not discover how TSE model works on continuous time slices from a match. This time we feed time slices generated from a match to the trained TSE model in chronological order and observe its performance on each match. Figure \ref{fig:resm} shows the performance of TSE model on four sample test matches. The horizontal axis $t$ represents the remaining time before the end of current match, which declines as the match goes on. The vertical line in the middle splits the graphs into two parts. The right side of it is the last 50\% time of the game, which this TSE model is trained to fit. 

\begin{figure}[htb!]
\centering
\subfigure[]{
\label{fig:result1}
\includegraphics[width=5.5cm]{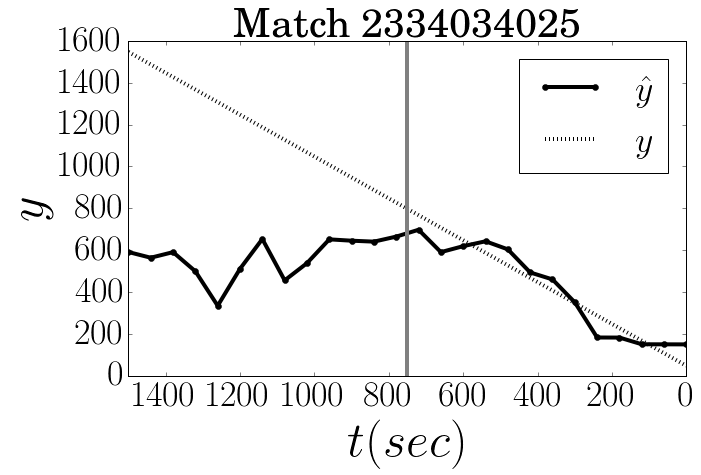}}
\subfigure[]{
\label{fig:result2}
\includegraphics[width=5.5cm]{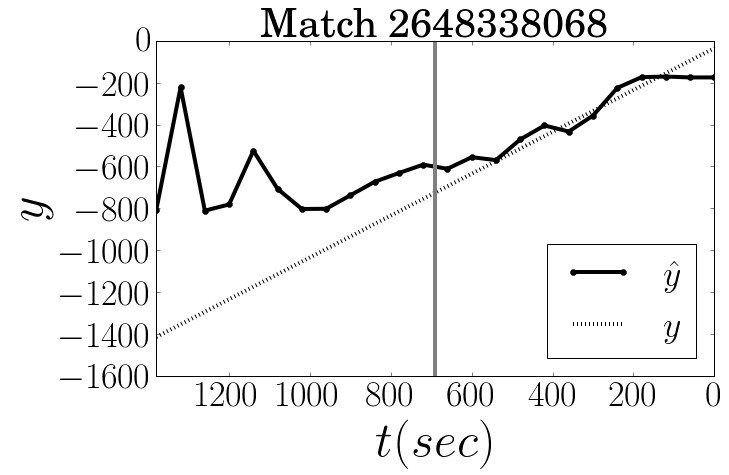}}
\subfigure[]{
\label{fig:result3}
\includegraphics[width=5.5cm]{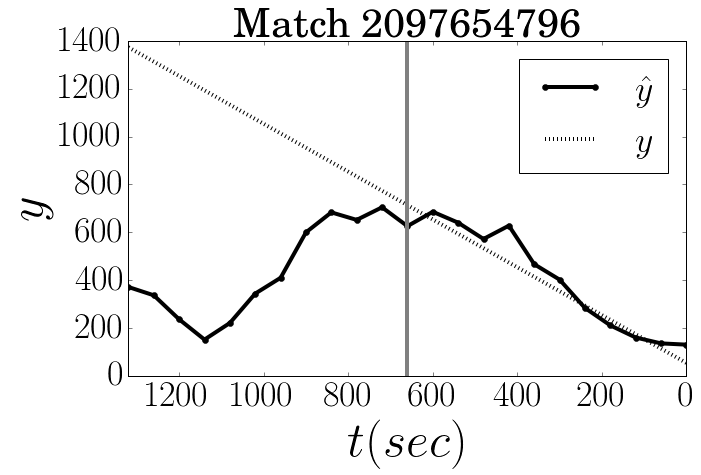}}
\subfigure[]{
\label{fig:result4}
\includegraphics[width=5.5cm]{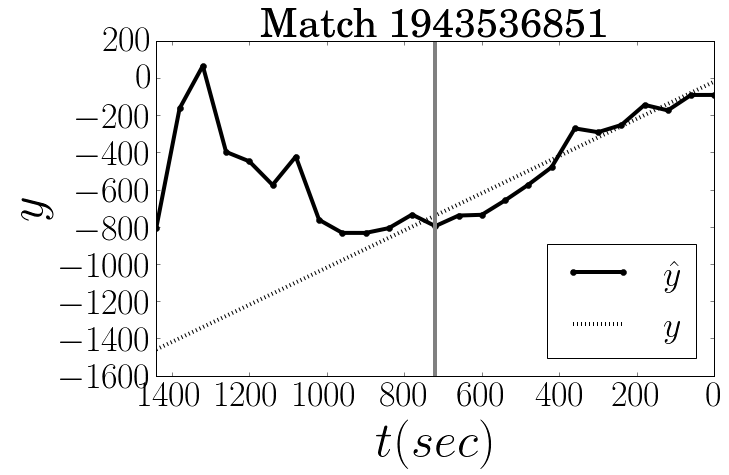}}
\caption{Performance on Sample Matches}
\label{fig:resm}
\end{figure}

Focusing on the right side, we can see that TSE model successfully fits $\mathit{DE_{TS}}$ in the last half of matches. It works well on matches of arbitrary lengths no matter Radiant or Dire wins. The predicted value usually fluctuates around the real value. The missing variables such as skills of heroes can be a major reason for such fluctuation. 

The observations on whole matches above provide proof of the effectiveness of MOBA-Slice on DotA2. As we can see from the figures, TSE model trained with data of last half matches can evaluate time slices from the last half of new matches thoroughly. This result shows that the functional relationship between the state of a time slice and the game result described in $\mathit{DE_{TS}}$ indeed exists, or TSE model would never be able to work efficiently on new matches. With the figures and the properties of $\mathit{DE_{TS}}$ explained previously, we can see that TSE model can do more than just evaluation. It is also useful enough to work as an online prediction model for the game result and the remaining time.

\subsection{Comparison with DotA Plus Assistant on Result Prediction}
In the progress of our research, a monthly subscription service called DotA Plus\footnote{Homepage of DotA Plus: \url{https://www.dota2.com/plus}} was unveiled in an update of DotA2 on March 12, 2018. As one part of DotA Plus service, Plus Assistant provides a real-time win probability graph for every match, as shown in Figure \ref{fig:plus}. It is praised as the "Big Teacher" by Chinese players due to its accurate prediction. Although Plus Assistant can not predict the remaining time of a match, its prediction of game result can be used to evaluate the corresponding ability of MOBA-Slice.

\begin{figure}
\centering
\includegraphics[width=9cm]{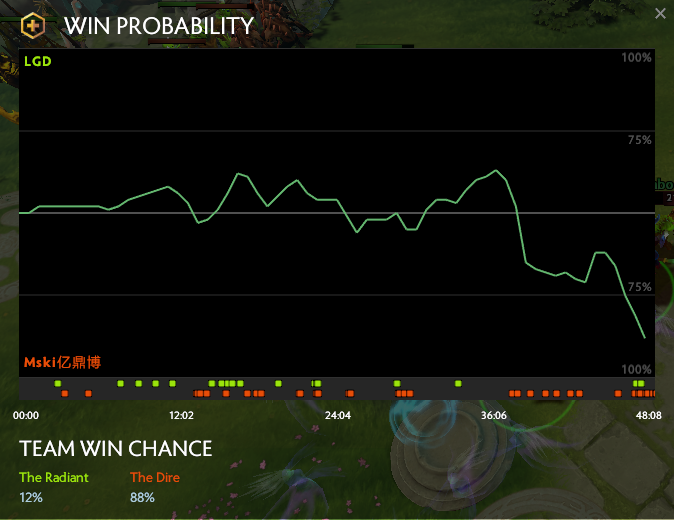}
\caption{Win Probability Graph of DotA Plus Assistant}
\label{fig:plus}
\end{figure}

To compare DotA Plus Assistant with MOBA-Slice, we have to choose matches played after its release date, so we can not use the previous dataset. The tournament of DotA2 Asian Championship\footnote{Homepage of DotA2 Asian Championship: \url{http://www.dota2.com.cn/dac/2018/index/?l=english}} was held from March 29th to April 7th. This tournament has 105 matches totally, but 33 matches are unusable for network failure or newly introduced heroes that are not included in the training data of our TSE Model. An upsetting fact is that the winning probability data is not available from the replay parser, and is suspected to be encrypted. So we manually download the valid 72 matches in DotA2 client and watch their replay to capture the winning probability graph. We discretize the winning probability into 3 values: Radiant's victory when above the middle line, Dire's when below it, and unknown when at it. We also feed these matches to the trained TSE model, and discretize the output according to the sign of value as property 1 of $DE_{TS}$ described: positive for Radiant's victory, negative for Dire's, and zero for unknown.

As we have the prediction from both models, we sample the prediction and compare with the real result to calculate the accuracy of a model at a specific game time percent. For example, the prediction results of MOBA-Slice on each match at the time point of 10\% of game time are used to calculate the accuracy of MOBA-Slice at 10\% time point. The unknown predictions are always counted as errors. We do this calculation every 10 percent of time for both models, and calculate the average for all time points. As the result in Table \ref{tab:plus}, the average accuracy of MOBA-Slice is 3.7\% higher than DotA Plus Assistant at predicting the game result.

\begin{table}[htb!]
\centering
\caption{Prediction Accuracy of MOBA-Slice and DotA Plus Assistant}
\begin{tabular}{|c|c|c|c|c|c|c|c|c|c|c|}
\hline
Game time percent & 10\% & 20\% & 30 \% & 40\% & 50\% & 60\% & 70\% & 80\% & 90\% & Average \\ \hline
DotA Plus Assistant & 0.4167 & 0.5139 & 0.5972 & 0.6111 & 0.6806 & 0.7500 & 0.7778 & 0.8472 & 0.9444 & 0.6821 \\ \hline
MOBA-Slice & 0.5694 & 0.5417 & 0.6111 & 0.7083 & 0.7083 & 0.7222 & 0.8056 & 0.8611 & 0.9444 & \underline{0.7191} \\ \hline
\end{tabular}
\label{tab:plus}
\end{table}

\subsection{Towards Better Performance}
\label{sec:tbp}
For the first 50\% time of the match in Figure \ref{fig:resm}, we can see that the error is much larger. One possible reason is that this TSE model is not trained to fit the first half of games, which contains different information from the last half. Changing the range of data is a way to fix this problem. On the other hand, larger instability at the beginning of matches also counts. The game situation may break our underlying assumption for the monotonicity of relative advantage between teams. We can not avoid such instability as the game situation may change at any time. But we can treat this as noise in the data. As long as the instability in training data is relatively small, the models are effective.

We want to see the disparity of instability in different parts of matches, so we train models with time slices from different intervals of game time. As TSE model takes a longer time to train due to the structure of Ind part, we choose to work with merely its Glo part. 

\begin{figure}
\centering
\includegraphics[width=10cm]{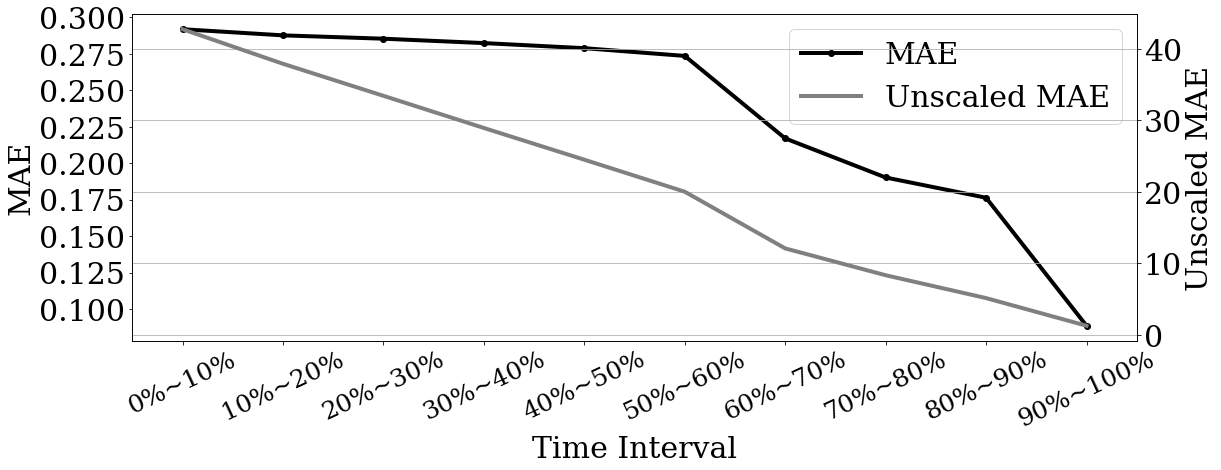}
\caption{Experiments on Different Intervals}
\label{fig:ex1}
\end{figure}

As shown in Figure \ref{fig:ex1}, experiments are carried out for every 10 percent of game time. We train with time slices from 0\%-10\% of game time and test on 0\%-10\%, and then repeat training and testing on 10\%-20\% and so on. The loss decreases as a match goes on, indicating the decline of instability. In other words, it is much more difficult to evaluate the beginning of a match than the ending. In interval 90\%-100\%, we can predict the game result with an average error of about 1 minute, which shows a high relevance of game situation and outcome. But we can not train models specific for an interval of game time to evaluate ongoing matches, as we do not know the length of a match until it ends. We further experiment on larger intervals to find out a proper range of training data for predicting arbitrary time slices.

Results in Figure \ref{fig:ex2} are much better than in Figure \ref{fig:ex1}, since relatively more training data is fed to the model. The same trend that loss decreases as the game progresses is seen in Figure \ref{fig:ex2}. However, if observed from right to left, Figure \ref{fig:ex2} shows something contrary to the common belief that training with more data results in better performance. We find that as time interval gets larger, the model learns from more time slices but the loss keeps growing. The instability of the beginning part seems so large as to worsen the performance. It appears to be a trade-off problem to choose the proper range of training data. The MAE in Figure \ref{fig:ex2} cannot be used to determine this as they are calculated on different test sets. We suppose the best choice needs to be found in accordance with the performance of MOBA-Slice in actual applications.

\begin{figure}
\centering
\includegraphics[width=10cm]{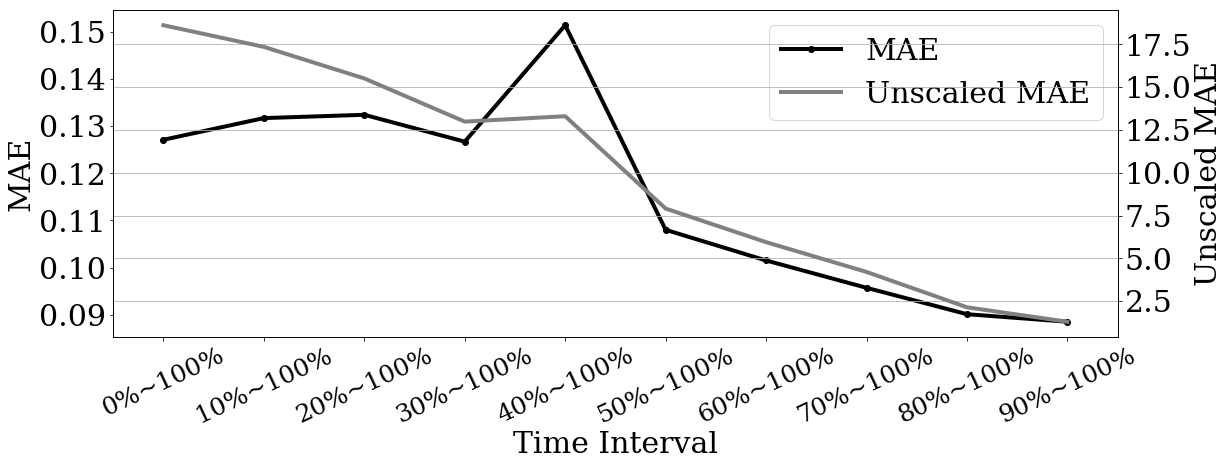}
\caption{Experiments on Larger Intervals}
\label{fig:ex2}
\end{figure}

\section{Related Work}
With the continuously growing number of players, MOBA games have been popular with young people all over the world. As the design of MOBA games is complicated, the large domain of knowledge contained in these games is of great research value. Research related to MOBA games mainly consists three aspects: strategy analysis, result prediction, and AI developing.

Strategy analysis is proved efficient at improving the performance of professional players. Yang and Roberts~\cite{Yang2013Knowledge} introduced a data-driven method to do post-competition analysis and knowledge discovery, and evaluated on 3 games: DotA, Warcraft \uppercase\expandafter{\romannumeral3} and Starcraft \uppercase\expandafter{\romannumeral2}. Hao et al.~\cite{Hao2015Player} studied player behavior and optimal team positions through clustering. Bauckhage's team~\cite{Bauckhage2014Beyond} introduced a spatio-temporal clustering method to partition game maps into meaningful areas. Cavadenti's group~\cite{Cavadenti2016What} implemented a data mining based method that discovers strategic patterns from historical behavioral traces to help players improve their skills. Hanke and Chaimowicz~\cite{AIIDE1715902} developed a recommendation system for hero line-ups based on association rules. In most strategy related researches, evaluation of a game situation or the game result is used as a part of the algorithm to tell the effectiveness of a strategy or to optimize toward a bigger winning rate. It is hard for such evaluation to fully utilize the multitudinous variables of MOBA games. In the research of strategy analysis, we also met the problem of evaluation of strategy. This ended up with the proposal of MOBA-Slice. MOBA-Slice is designed to provide reliable quantitative evaluation of game situations. We now suggest using MOBA-Slice as the evaluation algorithm in strategy analysis researches.

It is considered hard to predict the winners of a MOBA game, as the game situation is always changeful. There are so many factors and features to choose from to build a prediction model. Conley and Perry~\cite{ConleyHow} implemented K-Nearest Neighbor (KNN) algorithm to predict with merely hero lineup information. Wang~\cite{Wang2016Predicting} further used multi-layer feedforward neural networks and added game length information into inputs, but did not improve significantly. Almeida's team~\cite{Almeida2017Prediction} built classification models of Naive Bayes, KNN, and J48 based on the composition of heroes and the duration of match. Pu et al.~\cite{Pu2014Identifying} also worked towards identifying predictive information of winning team with graph modeling and pattern extraction. Different from static prediction in previous works, DotA Plus Assistant supports real-time prediction, but its algorithm remains business secret. As a nice property of MOBA-Slice, it supports real-time prediction of game results. And MOBA-Slice demonstrates better accuracy than DotA Plus Assistant.

Artificial intelligence of MOBA games interests many researchers. Many rule based bots and machine learning algorithms have been developed. Andersen et al.~\cite{Kresten2009EXPERIMENTS} examined the suitability of online reinforcement learning in real-time strategy games with Tank General as early as 2009. Synnaeve and Bessiere~\cite{Synnaeve2011A} used a semi-supervised method with expectation-maximization algorithm to develop a Bayesian model for opening prediction in RTS games. Kolwankar~\cite{Siddhesh2012Evolutionary} employed genetic algorithms to adjust actions in evolutionary AI for MOBA. Silva and Chaimowicz~\cite{Silva2016On} implemented a two-layered architecture intelligent agent that handles both navigation and game mechanics for MOBA games. Wisniewski and Miewiadomski~\cite{Wi2016Applying} developed state search algorithms to provide an intelligent behavior of bots. Pratama et al.~\cite{Pratama2017Fuzzy} designed AI based on fuzzy logic with dynamic difficulty adjustment to avoid static AI mismatching player and AI's difficulty level. AIIDE StarCraft AI competition~\cite{AIIDEStar} has been held every year since 2010, with continuously rising popularity. Evaluation of situation is often needed in artificial intelligence algorithms, which MOBA-Slice is designed to provide. MOBA-Slice can work as the reward function in reinforcement learning models~\cite{Sutton:1998:IRL:551283} and the evaluation function in Monte Carlo planning models~\cite{Chung2005Monte}.

\section{Conclusion}
MOBA-Slice is a well-designed framework which evaluates relative advantage between teams. It consists three parts: MGR analysis, discounted evaluation, and TSE model. With MGR analysis we manage to describe the deciding factors of MOBA game result. Discounted evaluation function $\mathit{DE_{TS}}$ has several appealing properties one of which is representing the relative advantage between teams. TSE model is designed to fit function $\mathit{DE_{TS}}$. Through applying MOBA-Slice to DotA2, we prove its effectiveness with experiments on a large number of match replays and comparison with DotA Plus Assistant. MOBA-Slice establishes a foundation for further MOBA related research which requires evaluation methods, including AI developing and strategy analysis.

Here are several aspects listed for future work. The part of the future trend in MGR analysis requires rating algorithm for players and teams. In the experiments, information of heroes' skills and the environment is not taken into consideration due to the current limitation of data processing. More MOBA games can be chosen to apply and test MOBA-Slice. Further, current TSE model is simply based on single time slice. Sequential prediction models can be designed with Recurrent Neural Network to take in time slice sequences, which will provide more information about the game situation.

\bibliographystyle{splncs04}
\bibliography{reference}

\end{document}